\documentclass[11pt]{article}

\usepackage[utf8]{inputenc}
\usepackage[T1]{fontenc}
\usepackage[margin=1in]{geometry}
\usepackage{amsmath,amssymb,amsfonts}
\usepackage{graphicx}
\usepackage{booktabs}
\usepackage{array}
\usepackage{xcolor}
\usepackage{caption}
\usepackage{siunitx}
\usepackage{enumitem}
\usepackage{authblk}
\usepackage{multirow} 
\usepackage{makecell}
\usepackage{amsthm}
\usepackage{bbm}
\usepackage{fontawesome5}
\usepackage{algorithm}
\usepackage{algpseudocode}
\usepackage{parskip}
\usepackage[colorlinks=true,linkcolor=blue,citecolor=blue,urlcolor=blue]{hyperref}
\usepackage{cleveref}
\usepackage{fontawesome5}
\usepackage{hyperref}

\graphicspath{{figures/}}

\newcommand{\xpresto}{\textsc{XPress}}
\newcommand{\R}{\mathbb{R}}

\usepackage{graphicx}
\newcommand{\hfemoji}{\raisebox{-0.15em}{\includegraphics[height=1em]{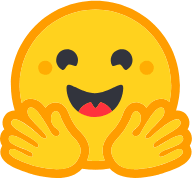}}}

\newtheorem{proposition}{Proposition}

\title{\textbf{\xpresto}: Parallel Refinement for Diffusion Drafters\\ in Speculative Decoding}

\author[1]{Zheng Wang}
\author[2]{Davis Wertheimer}
\author[2]{Yu Chin Fabian Lim}
\author[2]{Mudhakar Srivatsa}
\author[2]{Raghu K. Ganti}
\author[1]{Minjia Zhang}
\author[2]{Naigang Wang}
\affil[1]{University of Illinois Urbana-Champaign}
\affil[2]{IBM}

\date{%
  \href{https://github.com/Supercomputing-System-AI-Lab/xPress}{\faGithub~Code}
  \qquad
  \href{https://supercomputing-system-ai-lab.github.io/blogs/blog/xpress-parallel-refinement-for-diffusion-drafters-in-speculative-decoding/}{\faGlobe~Blog}
  \qquad
  \href{https://huggingface.co/UIUC-SSAIL}{\hfemoji~Hugging Face}
}

\begin{document}
\maketitle

\begin{abstract}
\noindent Block-diffusion drafters like dFlash generate an entire block of draft tokens in a single forward pass, drastically reducing the overhead of multiple-token drafting in speculative decoding. The crucial final step of the single-pass discrete denoising process involves using the logit distribution at each position to sample conditionally independent tokens. The resulting draft is thus a set of per-position marginals, rather than a joint distribution: no draft token is guaranteed to depend on its predecessors. Such independently sampled marginals tend to produce sequences with tokens that are individually likely, but jointly improbable under the target model's distribution, which verifies each token conditionally. This can cause early rejection and limits acceptance length. To address this, we propose \textbf{\xpresto{}} as a means to restore the missing causality in diffusion drafters. \xpresto{} is a lightweight causal refiner that reconciles the whole diffusion block at once through parallel refinement, restoring and propagating causal dependencies across the draft without a token-by-token loop. On Qwen3-8B, across seven math, code, and chat benchmarks, \xpresto{} raises acceptance length by ${\sim}30\%$ on average (up to $+56\%$) and its end-to-end decoding throughput by ${\sim}1.3\times$ on average (up to $1.7\times$) compared to the original dFlash diffusion drafter.
\end{abstract}

\begin{figure}[t]
  \centering
  \includegraphics[width=\linewidth]{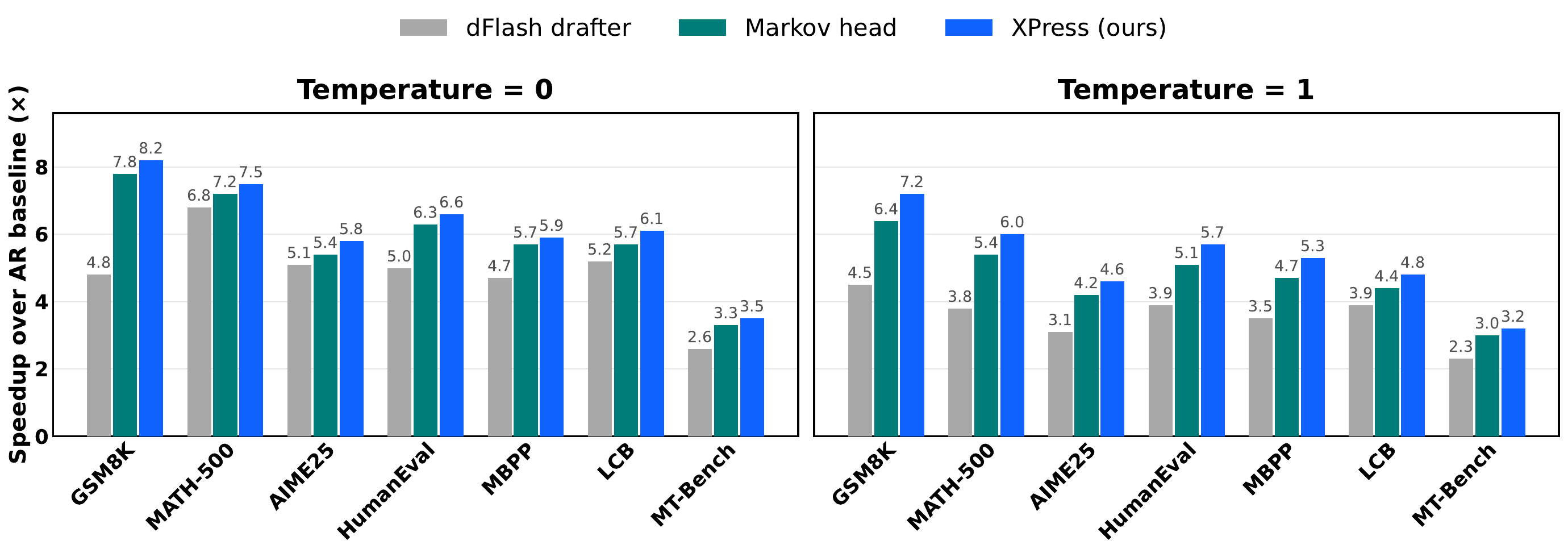}
  \caption{\textbf{End-to-end decoding throughput speedup over the autoregressive baseline} on Qwen3-8B with a block-16 dFlash drafter under greedy decoding ($T{=}0$, left) and lossless sampling ($T{=}1$, right). The throughput speedup values are measured on H200.  \xpresto{} consistently outperforms the Markov head and dFlash diffusion drafter.}
\label{fig:speedup-t0-t1}
\end{figure}
\section{Introduction}
Speculative decoding (SD)~\cite{chen2023acceleratinglargelanguagemodel, leviathan2023} accelerates autoregressive generation by using a lightweight draft model to propose future tokens, which the larger target model verifies in one parallel forward pass. 
A single multi-token verification pass costs about the same as a standard single-token target-model decoding step in low concurrency scenarios, so every draft token that matches target model outputs is another token produced at no additional target cost. The speedup achieved by SD is mainly governed by two factors: the acceptance length (the number of drafted tokens the target accepts per verification step) and the cost of the drafting process itself. A larger acceptance length amortizes each target verification pass over more generated tokens, while a cheaper drafter reduces the overhead paid to produce them. This forms a natural trade-off: we want to maximize the expected acceptance length of the drafter, while avoiding a proportional increase to drafting overhead.

\noindent Autoregressive (AR) generation has long been the default approach to drafting models, and the EAGLE series~\cite{eagle2024,li2024eagle2fasterinferencelanguage,li2025eagle3scalinginferenceacceleration} is one of the most representative AR-drafting SD methods. EAGLE's drafter is remarkably lightweight, as small as a single layer, yet it yields high-quality drafts. However, because drafting is autoregressive, generating $n$ draft tokens involves $n$ sequential forward passes of the draft model. As $n$ becomes larger, the drafting overhead becomes increasingly pronounced, but acceptance length, which requires an unbroken chain of accepted verifications, does not. \textbf{dFlash}~\cite{dflash2026} resolves this via a block-diffusion model that proposes an entire block of draft tokens in a single forward pass. By turning $n$ serial steps into a single parallel one, dFlash enables longer drafts (and downstream speedups for the target model) at near-constant overhead.

\noindent But the parallelism brought by the diffusion drafter carries an inherent limitation in accuracy. Unlike an AR drafter, where each position is conditioned on the preceding tokens, positions in a diffusion drafter are decoded from their marginal distributions. The token $k$ is drawn without seeing what the token $k{-}1$ turned out to be, so the block is a set of individually plausible tokens with no guarantee of causality. 
This can violate natural linguistic dependencies across positions, 
even when every token is locally high-probability~\cite{wu2025fastdllmv2efficientblockdiffusion,liu2025tidarthinkdiffusiontalk,presto2026}. For example, a drafter predicting each position independently can put a plural verb after a singular subject and produce ``she are,'' where each word is fine on its own but the verb contradicts the subject. At verification, these locally reasonable but jointly incoherent samples are rejected early by the left-to-right target model, limiting the achievable acceptance length.

\noindent One line of prior work tries to address this limitation by constructing a draft token tree. Tree-based drafting proposes a tree of candidate continuations and verifies the whole tree in a single target pass, so that the longest accepted path through the tree can be selected. Recent works, like PRESTO~\cite{presto2026} and DDTree~\cite{ddtree2026}, have demonstrated that tree drafting can effectively enhance the achievable acceptance length of diffusion-drafter-based SD methods. Nevertheless, tree drafting has real limitations. Fundamentally, it hedges around the non-causality of the diffusion drafter by targeting recall rather than accuracy. On top of that, the required sparse, irregular tree attention is expensive and complex to serve. Moreover, because a candidate tree typically spans tens to hundreds of tokens, its gains fade quickly at large batch sizes, where the target pass is already compute-bound and those extra tokens are no longer free. Therefore, we ask the following research question:

\emph{Can we cheaply inject causal information into the diffusion drafter at training time, so as to increase acceptance length and raise end-to-end decoding throughput without sacrificing its parallel drafting?}

\noindent What makes this plausible is a property already established for diffusion drafters: the correct token is frequently among the drafter's top-$k$ candidates at each position~\cite{presto2026}. The drafter's block-level marginals already narrow each position to a small candidate set, it just lacks the causal information needed to identify the right token within that set. The problem of correction therefore reduces to picking the right token from a narrow preexisting set. Therefore, we formalize this as a logits refinement problem and address it with \xpresto{}, a lightweight causal refiner that injects causal information into the diffusion drafter. Specifically, \xpresto{} keeps the diffusion drafter's parallel block proposal intact and reranks it in place: a small head reads the drafter's block hidden states together with the preceding tokens, and emits a per-position additive bias on the diffusion drafter's own logits. Naively, such causal refinement would reintroduce the very autoregressive dependency that diffusion drafting removes: each position conditions on the previous one, forcing a sequential sweep over the block. To avoid this, \xpresto{} resolves the refinement with Jacobi decoding: all positions are updated in parallel from the previous iterate, and $K\!\ll\!B$ such passes, where $B$ is the diffusion block size, suffice to reach the sequential fixed point in practice. We summarize the contributions of \xpresto{} as follows:

\begin{itemize}[leftmargin=*]
\item \textbf{Causal dependency injection at near-zero cost.} \xpresto{} injects causal information to diffusion drafter with a deliberately minimal mechanism: token embeddings and drafter hidden states are projected into a low-rank $r$-space, where a light-weight single \emph{strictly lower-triangular} linear mixer lets every position aggregate information from all of its predecessors. The resultant logits bias are then added to the diffusion drafter's own logits to do the causal refinement.
\item \textbf{Parallel refinement.} \xpresto{} resolves the refiner's causal dependency with $K$ parallel Jacobi iterations over the whole block instead of a conventional auto-regressive left-to-right decoding process. Empirically, these iterations converge in far fewer steps than the
diffusion block length ($K \ll B$), so a handful of batched forwards reproduce the effect of $B$ dependent auto-regressive steps while retaining the decoding parallelism.
\item \textbf{Consistent end-to-end gains.} On the Qwen3-8B target with the
dFlash drafter, \xpresto{} enhances acceptance length by 18--56\% under greedy decoding and by 19--58\% under $T{=}1$ sampling, across seven
math, coding, and chat benchmarks compared to the dFlash's diffusion drafter. The gains convert into end-to-end
decoding speedup: ${\sim}1.3\times$ over the drafter on average (up to
$1.7\times$) at $T{=}0$, growing to ${\sim}1.46\times$ on average (up to
$1.6\times$) at $T{=}1$ as shown in Figure ~\ref{fig:speedup-t0-t1}.
\end{itemize}

\section{Related Works}
\label{sec:related-work}
\subsection{Diffusion drafters}
A recent line of speculative decoding work pairs diffusion based parallel drafting with
autoregressive verification. One direction makes the drafter itself stronger: TiDAR~\cite{liu2025tidarthinkdiffusiontalk} trains a hybrid model that performs quadratic self-speculation to achieve parallel decoding, although its final generations are not lossless. DiffuSpec~\cite{li2025diffuspecunlockingdiffusionlanguage} and
SpecDiff-2~\cite{sandler2025specdiff2scalingdiffusiondrafter} employ large pretrained dLLMs as speculative drafters, with inference-time search or train–test
alignment to improve acceptance. However, these approaches rely on massive drafters (e.g., 7B parameters), which incur
substantial memory and latency overhead. On the contrary, dFlash~\cite{dflash2026} harnesses a small block-diffusion drafter conditioned on context
features of the target fills an entire block in one forward pass, reporting
over $6\times$ lossless acceleration. More recently, Nemotron-Labs-Diffusion~\cite{fu2026nemotronlabsdiffusiontrimodelanguagemodel} also demonstrates impressive decoding speedup in linear self-speculation decoding mode where the model itself can serve as a diffusion drafter as well as an AR target. A second direction improves how the drafter's marginals are spent at
verification stage. Because a diffusion drafter independently predicts a full distribution at every position in parallel, each position carries several plausible candidates, and committing to a single linear draft discards this breadth. PRESTO~\cite{presto2026}, DDTree~\cite{ddtree2026}, and JetSpec~\cite{hu2026jetspecbreakingscalingceiling} instead expand the per-position candidate space into a draft tree, enumerating multiple
continuations per position and verifying them jointly in one target pass, so that a mismatch at a single position no longer forces the entire block to be rejected. This raises the acceptance length without modifying or retraining the drafter. \xpresto{} is orthogonal to both directions: it
neither enlarges the drafter nor changes verification, but repairs the marginal distribution itself, injecting the intra-block causal
conditioning that one-pass diffusion drafting omits. Because it operates on
any block drafter's hidden states and logits, it composes with lighter or
heavier drafters alike, and tree methods built on refined logits inherit a
better backbone.

\subsection{Causal information injection for diffusion drafters}
Closest to us, two concurrent works, Domino~\cite{domino2026} and
DSpark~\cite{dspark2026}, share our goal of restoring causality to a diffusion
drafter by attaching a correction head that makes a position's logits depend
on the tokens before it. Domino walks the block left to right with a GRU,
carrying a recurrent state and emitting a per-position logit correction.
DSpark sheds the GRU's overhead with an even lighter Markov head. Both heads restore causal structure, but
both pay for it in ways \xpresto{} does not. First, both are serial:
Domino's GRU passes state position to position, and DSpark's bias for token
$k$ is indexed by the sampled id of token $k{-}1$, so a block of $B$ tokens
takes $B{-}1$ steps that must run in order. Second, both feed the correction narrow inputs.
DSpark conditions on a single previous token id, repairing local two-token
clashes but blind to the rest of the block and to the drafter's hidden
states, which are precomputed representations that encode far more than any token
id~\cite{truth2026}. Domino's GRU state carries more history, but still never
sees the whole block at once. \xpresto{} mixes over positions explicitly and,
despite its causal output structure, reads the entire block's drafter hidden
states and resolves
its causal dependency with $K$ parallel Jacobi passes.

\section{Background}
\label{sec:background}

\subsection{Speculative decoding}
\label{sec:bg-spec}

Speculative decoding (SD) accelerates the decoding process of a target model by pairing it with a cheap drafter. Each round, the drafter proposes a window of future tokens. The target scores the whole window in one forward pass and accepts the longest acceptable prefix plus one bonus token of its own. Under greedy decoding the accepted
tokens are identical to standard decoding, and under sampling at temperature $T>0$, rejection sampling against the draft distribution preserves the target distribution exactly. Let $\tau$ denote the expected number of tokens accepted per SD step,
the average per-token latency is
\begin{equation}
  \mathcal{L} \;=\; \frac{T_{\mathrm{draft}} + T_{\mathrm{verify}}}{\tau},
  \label{eq:latency}
\end{equation}
so speedup over the target's autoregressive latency mainly comes from two levers:
raising the acceptance length $\tau$, or lowering the drafting cost
$T_{\mathrm{draft}}$~\cite{dflash2026}. 

\subsection{Block-diffusion drafters}
\label{sec:bg-diffusion}

Block-diffusion drafters decouple the two levers. Let $c$ denote the
verified context (the accepted prefix so far) and consider a draft block
of $B$ positions indexed by $t \in \{1,\dots,B\}$, with $x_t$ the token at
position $t$, $x_{<t}$ its realized prefix within the block, and $|V|$ the
vocabulary size. Conditioned on $c$, the drafter consumes a masked block
and fills all $B$ positions in a single bidirectional forward pass,
producing logits $\ell_t \in \mathbb{R}^{|V|}$ for every position $t$.
Drafting cost is a single parallel forward pass, independent of $B$, in
contrast to the $B$ sequential passes required by auto-regressive (AR)
drafters. Therefore, the block can be long without inflating $T_{\mathrm{draft}}$. Unlike AR-based
drafters, diffusion-based drafters predict all positions independently and
in parallel, so $q_t(x_t \mid c) = \mathrm{softmax}(\ell_t)$ is a marginal
that conditions on $c$ but not on $x_{<t}$, rather than the target model's
auto-regressive conditional $p(x_t \mid c, x_{<t})$. The drafter's joint
proposal is therefore the factorized distribution
$\prod_{t=1}^{B} q_t(x_t \mid c)$, whereas the target it is verified
against, $\prod_{t=1}^{B} p(x_t \mid c, x_{<t})$, is path-conditioned.
Verification walks the block left to right and stops at the first token
inconsistent with the realized prefix. This is the inherent mismatch
between drafting and verification in diffusion-based SD
methods~\cite{presto2026}, which limits the achievable acceptance length.

\subsection{Jacobi decoding}
\label{sec:bg-jacobi}

Autoregressive decoding produces a block of $B$ tokens in $B$ dependent steps:
$y_t = \arg\max_y p(y \mid c, y_{<t})$ for $t = 1, \dots, B$. Jacobi
decoding~\cite{jacobi2021,santilli2023} reformulates this as solving a system
of $B$ equations in $B$ unknowns, $y_t - \arg\max_y p(y \mid c, y_{<t}) = 0$,
by fixed-point iteration: from an initial guess $y^{(0)}$, all positions are
updated simultaneously,
\begin{equation}
  y_t^{(k+1)} \;=\; \arg\max_y \, p\!\left(y \mid c,\, y_{<t}^{(k)}\right),
  \qquad t = 1, \dots, B,
  \label{eq:jacobi}
\end{equation}
which costs one parallel forward pass of the model per iteration (a causal
attention mask evaluates all $B$ conditionals at once). The iteration stops
when $y^{(k)} = y^{(k-1)}$. Because the dependency structure is strictly
causal, position $1$ is exact after one iteration, position $2$ after two, and
the fixed point, which can be reached in at most $B$ iterations, is provably identical to
the greedy autoregressive output~\cite{jacobi2021}. In practice, however, Jacobi decoding applied directly to a language model yields little speedup: from a cold-start initialization, the iteration rarely settles more than one token per pass~\cite{santilli2023}. This also motivates training methods that explicitly further shorten the Jacobi
trajectory, such as consistency LLMs ~\cite{cllm2024}. It adds a consistency loss that, from any intermediate state on a Jacobi trajectory, pulls the model's per-position predictions directly toward the trajectory's fixed point, so that many tokens collapse to their converged values in fewer passes. 

\section{\xpresto: Parallel Refinement for Diffusion Drafters}

In this section, we first introduce the designed lightweight causal refiner that refines the diffusion drafter's per-position marginals under intra-block causal dependencies (\Cref{sec:refiner}). We then elaborate how to resolve the refiner's causal dependency in parallel via Jacobi decoding, so a block converges in a few joint iterations instead of a serial left-to-right pass (\Cref{sec:jacobi}). Furthermore, we describe how the refiner is co-trained with the drafter under an acceptance-oriented objective, together with a consistency term that aligns training with the self-conditioned inputs seen during the refinement process (\Cref{sec:training}).

\begin{figure}[th]
  \centering
  \includegraphics[width=0.95\linewidth]{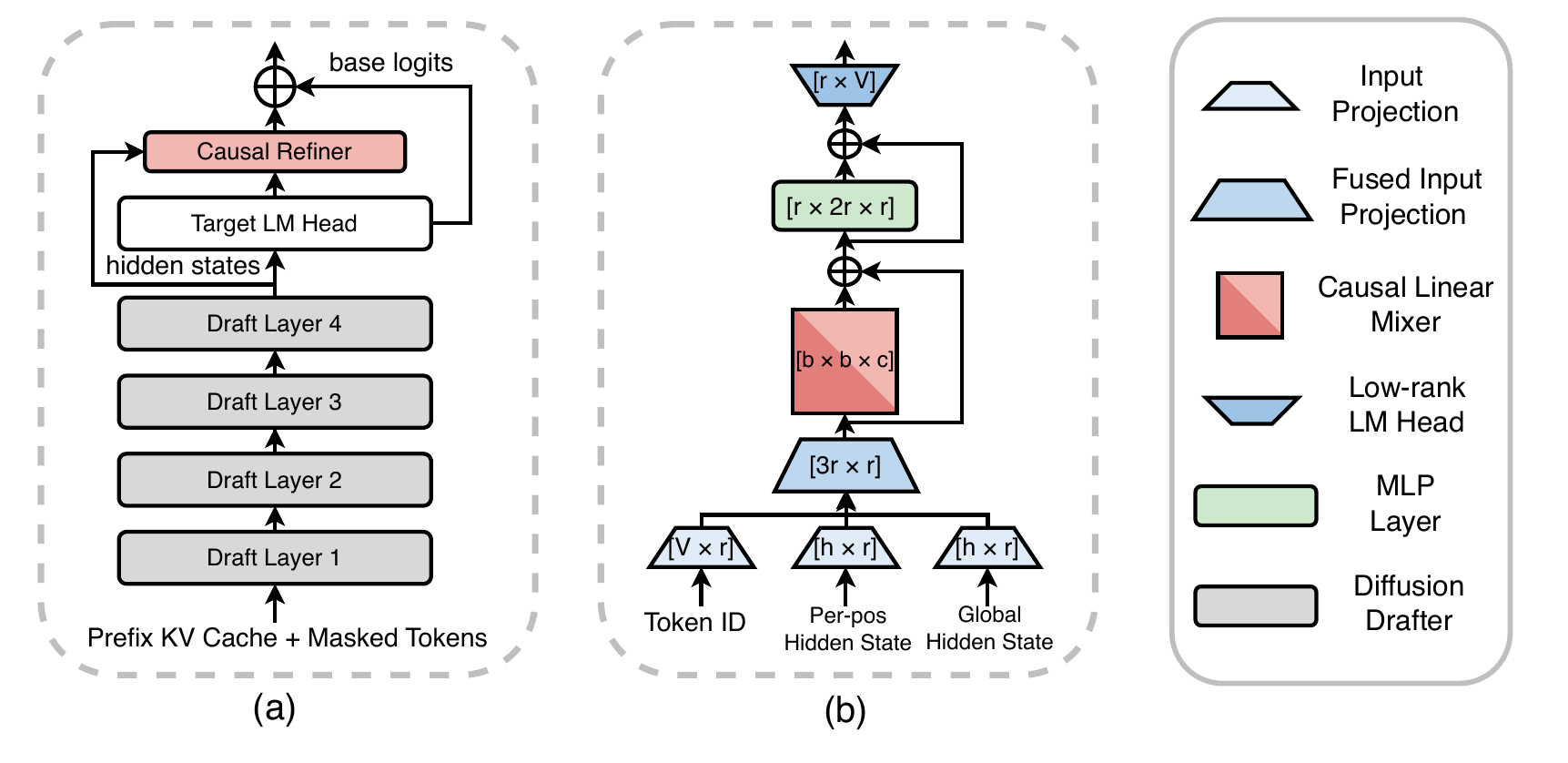}
  \vspace{-0.5cm}
  \caption{\textbf{(a)} The full pipeline. The block-diffusion drafter produces hidden states, the target LM head reads out the base logits, and the refiner adds a learned logits bias to form the final logits. \textbf{(b)} Inside the refiner. The three inputs (token id, global hidden state, and per-position hidden state) are down-projected into $r$-space, fused, mixed causally across block positions, passed through the $r$-space MLP, and read back out to vocabulary by the shared low-rank head.}
  \label{fig:arch}
\end{figure}


\subsection{A lightweight causal refiner}
\label{sec:refiner}
We first formulate the correction process as a causal refinement with four properties.
First, it should be \textbf{lightweight}: the drafter's single parallel forward
pass is already highly streamlined and performant, leaving little room for new
parameters or architectural complexity, and this existing capability should be
preserved. Second, within that small resource budget it should be
\textbf{causal}, injecting real causal information by conditioning each token on
its discretely sampled predecessors rather than merely smoothing the drafter's
marginals locally. Third, it should be \textbf{drafter-grounded}, making good
use of what the diffusion drafter already computes, such as its hidden states,
which typically carry richer information about the block than the pure token
id~\cite{truth2026}. Finally, it should incur \textbf{low overhead}: refinement
cost at inference should stay a small fraction of the total drafting time, and
in particular must avoid reintroducing a fully serial, left-to-right pass.

We design a lightweight causal refiner instantiating all four properties.
\Cref{fig:arch}(a) shows the full pipeline: the diffusion drafter proposes the
initial block in one pass, the target LM head reads out the base logits $s_k$,
and the refiner adds a learned correction on top. Consistent with the
aforementioned top-$k$ observation, the refiner does not score the vocabulary
from scratch. It adds a small per-position logit bias $\delta_k$ that re-ranks
the handful of candidates the drafter already favours. Formally, the refiner is
a map
\begin{equation}
  \mathcal{R}_\theta:\ \big(h_{1:B},\, x_{0:B-1}\big)\ \longmapsto\ (\delta_1,\dots,\delta_B)\in\R^{B\times V},
\end{equation}
where $h_{1:B}$ are the drafter's block hidden states and $x_{0:B-1}$ the
current token guesses ($x_0$ a fixed anchor). Composed with the base logits, it
induces a per-position refined distribution
\begin{equation}
  p^{r}_k(\cdot \mid x_{<k}, h_{1:B}) \;=\; \mathrm{softmax}\big(s_k + \delta_k\big),
  \qquad \delta_k = \mathcal{R}_\theta(h_{1:B}, x_{0:B-1})_k .
  \label{eq:refined-dist}
\end{equation}

\Cref{fig:arch}(b) details $\mathcal{R}_\theta$. Let $V$ be the vocabulary size,
$H$ the drafter's hidden width, $B$ the block length, and $r{=}256$ the low-rank
dimension. Using a column-vector convention, the
refiner is the composition
$\mathcal{R}_\theta=\mathrm{Read}\circ\mathrm{Trans}\circ\mathrm{Mix}\circ\mathrm{Fuse}$
of four typed stages. Write $A=[a_1,\dots,a_B]\in\R^{r\times B}$ and let
$a^{(d)}\in\R^{B}$ denote its $d$-th channel across positions:
\begin{align}
  \mathrm{Fuse}:\quad
    & a_k \;=\; W_{\text{in}}\big[\,W_h h_k \,\Vert\, W_g g \,\Vert\, W_e[x_{k-1}]\,\big] \;\in\R^{r},
    \label{eq:fuse}\\[2pt]
  \mathrm{Mix}:\quad
    & c^{(d)} \;=\; \big(I_B + L^{(d)}\big)\,a^{(d)} \;\in\R^{B},
      \quad L^{(d)}\in\R^{B\times B}\ \text{strictly lower-triangular},\ \forall d,
    \label{eq:mix}\\[2pt]
  \mathrm{Trans}:\quad
    & z_k \;=\; c_k + W_2\,\sigma\!\big(W_1 c_k\big) \;\in\R^{r},
      \quad W_1\in\R^{2r\times r},\ W_2\in\R^{r\times 2r},
    \label{eq:mlp}\\[2pt]
  \mathrm{Read}:\quad
    & \delta_k \;=\; W_r\, z_k \;\in\R^{V},
    \label{eq:read}\\[2pt]
  \mathrm{Correct}:\quad
    & \ell_k \;=\; s_k + \delta_k \;\in\R^{V}.
    \label{eq:correct}
\end{align}
The three inputs are the drafter's per-position hidden state $h_k\in\R^H$, the
previous-token id $x_{k-1}$, and a block-global summary
$g=\tfrac1B\sum_{j=1}^{B} h_j\in\R^H$ obtained by mean-pooling the drafter's
hidden states over the block; $\sigma$ is a pointwise nonlinearity. On the
output side, $s_k$ is the drafter's own base logit vector for position $k$,
$\delta_k$ the learned correction, and $\ell_k$ the corrected logit the block is
re-decoded from. Among the learned maps, $W_e\in\R^{r\times V}$ is the token
embedding, $W_h, W_g\in\R^{r\times H}$ the down-projections for $h_k$ and $g$,
$W_{\text{in}}\in\R^{r\times 3r}$ the input fusion projection,
$W_r\in\R^{V\times r}$ the readout head, and $L\in\R^{r\times B\times B}$ the
per-channel causal mixer whose channel slices $L^{(d)}$ are strictly
lower-triangular. Each design choice earns back one property elaborated as follows.

\textbf{Lightweight.} Everything except the two vocabulary matrices
  $W_e, W_r$ lives in $r$-space, and these are the embedding and prediction head
  required by any logit-bias model. On top of them the refiner adds only small
  $r$-space components, so by \eqref{eq:refined-dist} it acts as an additive
  correction to $s_k$ rather than overwriting the base drafter outputs.

\textbf{Drafter-grounded.} Feeding the per-position hidden $h_k$ and the
  block-global summary $g$ into the correction gives it strictly more signal
  than a bare token id~\cite{truth2026}, and both are obtained for free from the
  drafter.

\textbf{Causal.} By \eqref{eq:mix} position $k$ mixes only over its
  prefix $j\le k$, and because $\mathrm{Mix}$ acts on features that already carry
  the prior sampled tokens through $\mathrm{Fuse}$, this is genuine causal
  conditioning rather than a local smoothing of the marginals. The refined
  distribution inherits exactly this structure:
  \begin{proposition}[Causal in the sampled tokens]
  \label{prop:causal}
  For every position $k$, the correction $\delta_k$ depends on the sampled tokens
  only through the prefix $x_{<k}$. Consequently, each refined conditional depends
  only on the prefix, and the product
  \begin{equation}
    p^{r}(x_{1:B}\mid h_{1:B}) \;=\; \prod_{k=1}^{B} p^{r}_k\big(x_k \mid x_{<k},\, h_{1:B}\big)
  \end{equation}
  is the law of left-to-right sampling from the refiner, recovering the
  path-conditioned structure that the drafter's marginals $\prod_k q_k(x_k\mid c)$ lack.
\end{proposition}
  \noindent The claim is immediate from the dependency structure of \eqref{eq:fuse}--\eqref{eq:mix}:
Fuse reads only $x_{k-1}$, Mix aggregates $a_j$ only for $j \le k$, each depending
only on $x_{j-1}$, and Trans and Read act position-wise. Since $h_{1:B}$ and $g$ are
computed from the context and do not depend on the in-block tokens, $\delta_k$
depends on the tokens only through $x_{0:k-1}$.

\textbf{Low overhead.} Mix is a cheap $r$-space operation and Trans adds only an
$r \to 2r \to r$ residual MLP. Crucially, Fuse and Mix are both linear, so their
composition can fold into a single operator at inference. With
$\tilde a_j = \big[W_h h_j \,\|\, W_g g \,\|\, W_e[x_{j-1}]\big]$ we have
$c_k = \sum_{j \le k} M_{k,j}\, \tilde a_j$, where
$M_{k,j} = \big(\mathbbm{1}[j{=}k] + \mathbbm{1}[j{<}k]\,L^{(\cdot)}_{k,j}\big) \odot W_{\mathrm{in}}$
is precomputed once, with $L^{(\cdot)}_{k,j} \in \mathbb{R}^{r}$ collecting
$L^{(d)}_{k,j}$ over channels and $\odot$ scaling the $d$-th row of
$W_{\mathrm{in}}$ by its $d$-th entry, so the fuse--mix path costs a single matmul
per iteration.


\subsection{Parallel refinement via Jacobi decoding}
\label{sec:jacobi}
The causal mixer restricts the visibility of every position to its own prefix ($j\le k$), so the refiner naturally supports autoregressive generation. However, this left-to-right generation process pays an additional $B{-}1$-step loop cost over a block of $B$ draft tokens. Instead of finalizing one position before moving to the next, \xpresto{} updates all positions at once and repeats this a few times, correcting any prior mistakes via Jacobi decoding.

\begin{algorithm}[h]
\caption{Parallel refinement via Jacobi decoding (greedy; see text for $T>0$)}
\label{alg:jacobi}
\begin{algorithmic}[1]
\Require drafter hidden states $\{h_k\}_{k=1}^{B}$, global summary $g$, base logits $\{s_k\}_{k=1}^{B}$, anchor $y_0$, iterations $K$
\Ensure refined block $Y^{(K)}$
\State $Y^{(0)} \gets \big(\arg\max_v s_{k,v}\big)_{k=1}^{B}$ \Comment{seed from the drafter's one-shot predictions}
\For{$j = 0, 1, \dots, K-1$}
  \State $(\delta_1,\dots,\delta_B) \gets \textsc{Refiner}\big(Y^{(j)}, y_0, \{h_k\}, g\big)$ \Comment{one parallel forward; $\delta_k$ reads only $y_{<k}$}
  \For{$k = 1, \dots, B$ \textbf{in parallel}}
    \State $y^{(j+1)}_k \gets \arg\max_v \big(s_k + \delta_k\big)_v$
  \EndFor
\EndFor
\State \Return $Y^{(K)}$
\end{algorithmic}
\end{algorithm}

Formally, sequential greedy decoding computes $y_k = \arg\max_v p_k(v \mid y_{<k}, h, g)$ in order, whereas Jacobi decoding solves the same equations from a seed $Y^{(0)}$ by updating every position at once,
\begin{equation}
\label{eq:jacobi-update}
  y_k^{(j+1)} = \arg\max_v\; p_k\big(v \mid y^{(j)}_{<k},\, h, g\big), \qquad k=1,\dots,B.
\end{equation}

Each iteration updates every position from the block as it currently stands, and a position stops changing once the tokens before it have also halted. The detailed refinement process is shown in Algorithm~\ref{alg:jacobi}. The key structural fact enabling convergence of Jacobi decoding here is that the drafter features
$h,g$ are computed once and held fixed across iterations, so
\eqref{eq:jacobi-update} is a fixed-point iteration on the finite block
$V^{B}$ whose only per-position dependence is on the prefix $y_{<k}$. This makes
the outcome exact and cheap to bound.

\begin{proposition}[Finite convergence to the sequential decode]
\label{prop:converge}
Break ties in the $\arg\max$ of \eqref{eq:jacobi-update} by a fixed rule, and assume
$h, g$ are constant across iterations and the anchor $y_0$ is fixed. Let
$Y^\star = (y^\star_1, \dots, y^\star_B)$ be the sequential greedy decode,
$y^\star_k = \arg\max_v p_k(v \mid y^\star_{<k}, h, g)$. Then from any seed $Y^{(0)}$,
the Jacobi iteration \eqref{eq:jacobi-update} satisfies
\begin{equation}
  y^{(j)}_k = y^\star_k \quad \text{for all } j \ge k,
  \label{eq:monotone}
\end{equation}
so it reaches $Y^\star$ in at most $B$ iterations. Moreover $Y^\star$ is the unique
fixed point of \eqref{eq:jacobi-update}.
\end{proposition}

\begin{proof}
By the causal structure of the refiner (\Cref{prop:causal}), each
$p_k(\cdot\mid y_{<k},h,g)$ depends on the token guesses only through the prefix
$y_{<k}$; combined with $h,g$ and the anchor $y_0$ being fixed, the update at
position $k$ is a function of $y_{<k}$ alone. We now prove the invariant
``$y_k^{(j)}=y_k^\star$ for all $k\le j$'' by induction on $j$. The base case
$j{=}0$ is vacuous. Assume it holds after iteration $j$.
Fix any $k\le j{+}1$; then $k{-}1\le j$, so by the hypothesis and the fixed
anchor, $y_{<k}^{(j)}=y_{<k}^\star$. Since $h,g$ are fixed, the update gives
\[
  y_k^{(j+1)}
  =\arg\max_v p_k\!\big(v\mid y_{<k}^{(j)},h,g\big)
  =\arg\max_v p_k\!\big(v\mid y_{<k}^\star,h,g\big)
  =y_k^\star,
\]
establishing the invariant for $j{+}1$ and hence \eqref{eq:monotone}. Taking
$j{=}B$ gives $Y^{(B)}=Y^\star$. For uniqueness, any fixed point $Y$ obeys
$y_k=\arg\max_v p_k(v\mid y_{<k},h,g)$ for all $k$; induction on $k$ (position
$1$ depends only on the fixed anchor and constants, then each $y_k$ is determined
by the already-pinned prefix) forces $Y=Y^\star$.
\end{proof}

Intuitively, each iteration pins one more leading position. Once a prefix of
length $n$ is correct it never changes, and it makes position $n{+}1$ correct on
the next pass, so stability spreads rightward from the anchor. Proposition~\ref{prop:converge}
is thus the worst case; in practice the drafter's seed is already a good guess,
so many positions settle at once and $K$ iterations lock in far more than $K$
tokens. At $T>0$, each $\arg\max$ is replaced by Gumbel-max sampling with noise drawn once
per block and shared across iterations, and the verifier receives the refined
distribution of the final iteration as the proposal, so modified rejection sampling
remains lossless for any $K \ge 1$. We find $K \approx 6$--$7$ iterations suffice to match the acceptance length
of a $15$-step sequential decode.

\subsection{Training \xpresto}
\label{sec:training}
The causal refiner is co-trained with the (co-adapted) drafter, using ground-truth token sequences and the predictions of the frozen target model. Both provide useful training signal, that we incorporate into two separate loss terms. 
The first is a teacher-forced cross-entropy against the ground truth token sequence. Conditioning each position on the ground-truth prefix $y_{<k}$, it maximizes the probability of the correct token $x_k^*$. This is the standard next-token training objective that instills language capability into the refiner. But it optimizes the data likelihood, whereas what sets the speedup of SD is the acceptance rate. Under speculative sampling, the probability that a token drawn from vocabulary distribution $p$ is accepted against the target model distribution $p^t$ can be expressed as $\sum_x \min\!\big(p(x),\,p^t(x)\big) = 1 - \mathrm{TV}(p, p^t)$, with $\mathrm{TV} = \tfrac12\lVert p - p^t\rVert_1$ ~\cite{leviathan2023}. So
\begin{equation}
  \lVert p_k - p_k^t \rVert_1 \;=\; 2\,\mathrm{TV} \;=\; 2\,(1 - \text{accept rate}_k),
\end{equation}
and minimizing this total-variation distance to the target is exactly maximizing acceptance, so we make it the second loss term in our training objective. The two are complementary: the cross-entropy points the refiner at the right token, while the total-variation term shapes the distribution to the target the way acceptance is scored. The per-position loss is their weighted sum, with $w_k=\exp(-(k{-}1)/\gamma)$ emphasizing earlier positions \cite{dflash2026}, since an inference-time verification mismatch disqualifies not just that position but all following positions in the draft:
\begin{equation}
  \mathcal{L}(p) = \sum_{k} w_k \Big[ \alpha_{\text{ce}}\big(\!-\!\log p_k(x_k^*)\big) + \alpha_{\ell_1}\lVert p_k - p_k^t \rVert_1 \Big].
\end{equation}

During a forward pass, the drafter produces a base distribution $p^b$, which the refiner then uses to produce the refined distribution $p^r$. A naive application of our loss to $p^r$ yields a performant refiner, and a drafter co-adapted to its behavior. Yet this can be problematic at inference time: the Jacobi iteration begins from the drafter's predictions, so if $p^b$ drifts from the target $p^t$, the refiner is starting from a worse seed. We cannot differentiate through the token-sampling operation in the drafter that captures this dynamic, so we instead add an auxiliary loss on $p_b$, anchoring it to desired behavior. 

An additional concern is the fact that minimizing the loss on $p^r$ in a teacher-forced setting yields a refiner that is good at correcting gold prefixes, yet untested on the self-conditioned inputs it actually receives in the Jacobi decoding process. Therefore, we also introduce a consistency loss to solve this misalignment between training and inference stages. Refiner forward passes are cheap by design, so during training we run a second forward pass whose token inputs come from the drafter's $\text{argmax}(p^b)$, and apply our two-term loss to that output $\hat p^r$ as well.

The full objective thus applies $\mathcal{L}$ to three distributions:
\begin{equation}
  \mathcal{L}_{\text{total}} = \underbrace{\mathcal{L}(p^{r})}_{\text{refiner}} + \lambda\,\underbrace{\mathcal{L}(p^{b})}_{\text{drafter anchor}} + \beta\,\underbrace{\mathcal{L}(\hat p^{r})}_{\text{consistency}}.
\end{equation}
The first loss optimizes the refiner under teacher forcing, the second drafter-anchor loss preserves the quality of the diffusion drafter itself, and the final consistency loss allows the refiner to better operate under inference conditions.


\section{Experiments}
\subsection{Evaluation Settings}
\paragraph{Models and Datasets.}
We evaluate \xpresto{} on Qwen3-8B~\cite{yang2025qwen3technicalreport} as the target model, using
its non-thinking mode throughout for efficient decoding, with the released
block-16 dFlash drafter~\cite{dflash2026} as the base diffusion drafter. For
training data we use Open-PerfectBlend~\cite{xu2024perfectblendredefiningrlhf}, an open
instruction mixture spanning math, coding, and chat. Following the dFlash
recipe, assistant responses are regenerated by the target model under its chat
template in non-thinking mode, so the refiner is trained on the distribution it
must correct at inference. We evaluate on math benchmarks including
GSM8K~\cite{cobbe2021trainingverifierssolvemath}, MATH-500~\cite{lightman2023letsverifystepstep}, and AIME25~\cite{MAA_AIME_nd_a}; coding benchmarks
including HumanEval~\cite{chen2021evaluatinglargelanguagemodels}, MBPP~\cite{austin2021programsynthesislargelanguage}, and
LiveCodeBench~\cite{jain2024livecodebenchholisticcontaminationfree}; and open-ended conversational tasks including
MT-Bench~\cite{zheng2023judgingllmasajudgemtbenchchatbot}. None of the evaluation sets overlap the training
corpus.

\paragraph{Baselines and Training Settings.}
We compare against the plain dFlash drafter~\cite{dflash2026} and DSpark's
sequential Markov head~\cite{dspark2026}. The dFlash baseline measures the drafter as
released, with no correction head. We report two
metrics: the \emph{acceptance length} $\tau$, the average number of tokens
committed per target verification; \emph{end-to-end decoding throughput speedup}, the ratio of speculative throughput to the autoregressive baseline of the same target. All heads are co-trained with the drafter against the frozen target with AdamW, learning rate $6\times10^{-4}$. The diffusion block size is set to 16. The consistency loss on the drafter-seeded second pass is
weighted by $\beta{=}0.3$. The anchor term that holds the base drafter to
the target is annealed from $\lambda{=}0.6$ to a floor of $0.2$ over training,
shifting capacity toward the refiner once the drafter has adapted. For a controlled head-to-head comparison,
we train the Markov head under the identical setup as \xpresto{}: the
same co-training recipe, the same data, the same loss schedule, and the same step budget, so
the two systems differ only in the correction architecture.

\paragraph{Implementation Details.} For the single-batch evaluation, we use HuggingFace-based harness with
\texttt{torch.compile} on all modules and CUDA-graph capture of the correction
rollout. For serving-engine evaluation, we integrate \xpresto{} into
vLLM's speculative-decoding stack \cite{kwon2023efficientmemorymanagementlarge} to evaluate \xpresto{} performance on large batch sizes. Because GPU clocks on a shared cluster drift by several
percent, every throughput number is the mean of five measured rounds. For evaluation, we use a single H200. For the training process, we use 32 H100 GPUs to train each setting. 

\subsection{Results}

\begin{table*}[!t]
\centering
\small

\setlength{\tabcolsep}{2pt}
\caption{Decoding speedup over the autoregressive baseline (Sp.) and average
acceptance length ($\tau$) on Qwen3-8B with a dFlash block-16 drafter,
single-sequence decoding, at most 2048 generated tokens. At $T{=}1$, $\tau$ is averaged over 5 seeds.
\textbf{$\times$ Gain} is \xpresto{}'s throughput over the plain dFlash diffusion drafter.}
\label{tab:main-results}
\resizebox{\linewidth}{!}{%
\begin{tabular}{
l
@{\hspace{0.3em}} cc cc cc
@{\hspace{0.3em}} cc cc cc
@{\hspace{0.3em}} cc
@{\hspace{0.3em}} cc
}
\toprule \hline
\multirow{2}{*}{Method}
& \multicolumn{6}{c@{\hspace{1.2em}}}{\sc{Math}}
& \multicolumn{6}{c@{\hspace{1.2em}}}{\sc{Code}}
& \multicolumn{2}{c@{\hspace{1.2em}}}{\sc{Chat}}
& \multicolumn{2}{c}{\textit{Avg.}} \\
\cmidrule(lr){2-7} \cmidrule(lr){8-13} \cmidrule(lr){14-15} \cmidrule(lr){16-17}
& \multicolumn{2}{c}{GSM8K}
& \multicolumn{2}{c}{MATH-500}
& \multicolumn{2}{c@{\hspace{1.2em}}}{AIME25}
& \multicolumn{2}{c}{HumanEval}
& \multicolumn{2}{c}{MBPP}
& \multicolumn{2}{c@{\hspace{1.2em}}}{LCB}
& \multicolumn{2}{c@{\hspace{1.2em}}}{MT-Bench}
& \multicolumn{2}{c}{\textit{Avg.}} \\
\midrule
\multicolumn{1}{c}{Temperature = 0}
& Sp. & $\tau$ & Sp. & $\tau$ & Sp. & $\tau$
& Sp. & $\tau$ & Sp. & $\tau$ & Sp. & $\tau$
& Sp. & $\tau$
& Sp. & $\tau$ \\
\midrule
dFlash
& 4.8$\times$ & 6.48 & 6.8$\times$ & 7.71 & 5.1$\times$ & 7.10
& 5.0$\times$ & 6.44 & 4.7$\times$ & 5.75 & 5.2$\times$ & 7.11
& 2.6$\times$ & 3.18
& 4.9$\times$ & 6.25 \\
Markov head
& 7.8$\times$ & 9.67 & 7.2$\times$ & 9.24 & 5.4$\times$ & 7.95
& 6.3$\times$ & 7.76 & 5.7$\times$ & 6.90 & 5.7$\times$ & 7.85
& 3.3$\times$ & 4.13
& 5.9$\times$ & 7.64 \\
\xpresto{} (ours)
& \textbf{8.2$\times$} & \textbf{10.11} & \textbf{7.5$\times$} & \textbf{9.62} & \textbf{5.8$\times$} & \textbf{8.35}
& \textbf{6.6$\times$} & \textbf{8.15} & \textbf{5.9$\times$} & \textbf{7.11} & \textbf{6.1$\times$} & \textbf{8.40}
& \textbf{3.5$\times$} & \textbf{4.38}
& \textbf{6.2$\times$} & \textbf{8.02} \\
\cmidrule(lr){1-17}
\textbf{$\times$ Gain}
& \textbf{1.70$\times$} &
& \textbf{1.10$\times$} &
& \textbf{1.12$\times$} &
& \textbf{1.32$\times$} &
& \textbf{1.25$\times$} &
& \textbf{1.17$\times$} &
& \textbf{1.36$\times$} &
& \textbf{1.29$\times$} & \\
\midrule
\multicolumn{1}{c}{Temperature = 1}
& Sp. & $\tau$ & Sp. & $\tau$ & Sp. & $\tau$
& Sp. & $\tau$ & Sp. & $\tau$ & Sp. & $\tau$
& Sp. & $\tau$
& Sp. & $\tau$ \\
\midrule
dFlash
& 4.5$\times$ & 5.83 & 3.8$\times$ & 5.71 & 3.1$\times$ & 4.35
& 3.9$\times$ & 5.40 & 3.5$\times$ & 4.83 & 3.9$\times$ & 5.58
& 2.3$\times$ & 2.90
& 3.6$\times$ & 4.94 \\
Markov head
& 6.4$\times$ & 8.76 & 5.4$\times$ & 7.60 & 4.2$\times$ & 6.22
& 5.1$\times$ & 6.99 & 4.7$\times$ & 6.39 & 4.4$\times$ & 6.43
& 3.0$\times$ & 3.97
& 4.7$\times$ & 6.62 \\
\xpresto{} (ours)
& \textbf{7.2$\times$} & \textbf{9.20} & \textbf{6.0$\times$} & \textbf{8.12} & \textbf{4.6$\times$} & \textbf{6.68}
& \textbf{5.7$\times$} & \textbf{7.36} & \textbf{5.3$\times$} & \textbf{6.68} & \textbf{4.8$\times$} & \textbf{6.65}
& \textbf{3.2$\times$} & \textbf{4.13}
& \textbf{5.3$\times$} & \textbf{6.97} \\
\cmidrule(lr){1-17}
\textbf{$\times$ Gain}
& \textbf{1.60$\times$} &
& \textbf{1.58$\times$} &
& \textbf{1.48$\times$} &
& \textbf{1.46$\times$} &
& \textbf{1.51$\times$} &
& \textbf{1.23$\times$} &
& \textbf{1.39$\times$} &
& \textbf{1.46$\times$} & \\
\hline
\bottomrule
\end{tabular}%
}

\vspace{1.5em}   

\setlength{\tabcolsep}{5pt}
\caption{Serving-engine results on vLLM (Qwen3-8B target, greedy decoding, 2048
max generated tokens, EOS-terminated): throughput (tok/s) with speedup over the
autoregressive baseline below. $\tau$ varies by ${\leq}4\%$ across batch sizes,
and we report its mean. MT-Bench is evaluated over all turns, with each
instance's second turn conditioned on the model's own first-turn reply.}
\label{tab:vllm-batch}
\begin{tabular}{c l ccccc}
\toprule \hline
& Method & GSM8K & MATH-500 & HumanEval & MBPP & MT-Bench \\
\midrule
\multirow{2}{*}{$\tau$}
& dFlash     & 6.02 & 7.18 & 6.38 & 5.70 & 3.39 \\
& \xpresto{} & \textbf{9.82} & \textbf{9.62} & \textbf{8.15} & \textbf{7.39} & \textbf{4.47} \\
\midrule
\multirow{3}{*}{bs = 4}
& AR   & 752 & 756 & 754 & 748 & 729 \\
& dFlash     & \makecell{3087\\4.1$\times$} & \makecell{3811\\5.0$\times$} & \makecell{3401\\4.5$\times$} & \makecell{3055\\4.1$\times$} & \makecell{1842\\2.5$\times$} \\
& \xpresto{} & \makecell{\textbf{4540}\\\textbf{6.0$\times$}} & \makecell{\textbf{4610}\\\textbf{6.1$\times$}} & \makecell{\textbf{3933}\\\textbf{5.2$\times$}} & \makecell{\textbf{3728}\\\textbf{5.0$\times$}} & \makecell{\textbf{2318}\\\textbf{3.2$\times$}} \\
\midrule
\multirow{3}{*}{bs = 8}
& AR   & 1377 & 1411 & 1413 & 1415 & 1307 \\
& dFlash     & \makecell{5518\\4.0$\times$} & \makecell{6728\\4.8$\times$} & \makecell{5964\\4.2$\times$} & \makecell{5542\\3.9$\times$} & \makecell{3021\\2.3$\times$} \\
& \xpresto{} & \makecell{\textbf{8087}\\\textbf{5.9$\times$}} & \makecell{\textbf{8081}\\\textbf{5.7$\times$}} & \makecell{\textbf{7177}\\\textbf{5.1$\times$}} & \makecell{\textbf{6418}\\\textbf{4.5$\times$}} & \makecell{\textbf{3837}\\\textbf{2.9$\times$}} \\
\midrule
\multirow{3}{*}{bs = 16}
& AR   & 2675 & 2801 & 2767 & 2695 & 2187 \\
& dFlash     & \makecell{8020\\3.0$\times$} & \makecell{9435\\3.4$\times$} & \makecell{8767\\3.2$\times$} & \makecell{7582\\2.8$\times$} & \makecell{4085\\1.9$\times$} \\
& \xpresto{} & \makecell{\textbf{11795}\\\textbf{4.4$\times$}} & \makecell{\textbf{12084}\\\textbf{4.3$\times$}} & \makecell{\textbf{10220}\\\textbf{3.7$\times$}} & \makecell{\textbf{9139}\\\textbf{3.4$\times$}} & \makecell{\textbf{5141}\\\textbf{2.4$\times$}} \\
\midrule
\multirow{3}{*}{bs = 32}
& AR   & 4806 & 5179 & 5174 & 4878 & 3358 \\
& dFlash     & \makecell{9380\\2.0$\times$} & \makecell{10597\\2.0$\times$} & \makecell{10364\\2.0$\times$} & \makecell{8933\\1.8$\times$} & \makecell{4647\\1.4$\times$} \\
& \xpresto{} & \makecell{\textbf{13686}\\\textbf{2.8$\times$}} & \makecell{\textbf{13483}\\\textbf{2.6$\times$}} & \makecell{\textbf{12064}\\\textbf{2.3$\times$}} & \makecell{\textbf{10474}\\\textbf{2.1$\times$}} & \makecell{\textbf{5719}\\\textbf{1.7$\times$}} \\
\hline \bottomrule
\end{tabular}

\end{table*}

\textbf{Single Batch.} \Cref{tab:main-results} reports decoding speedup (Sp.) and average acceptance length
($\tau$) on Qwen3-8B with a dFlash block-16 drafter across math, code, and chat
benchmarks under single batch setting. \xpresto{} attains the best average in every setting. Under greedy
decoding ($T{=}0$) it reaches a $6.2\times$ average speedup at $8.02$ accepted
tokens per step, improving over the plain dFlash drafter ($4.9\times$, $6.25$)
by $1.29\times$ on average and up to $1.70\times$ on GSM8K, while matching or
exceeding the Markov-head refiner on every benchmark and consistently so in
acceptance length (e.g.\ $8.40$ vs.\ $7.85$ on LCB). The gap widens sharply once sampling is turned on: under
$T{=}1$ the plain diffusion drafter degrades severely, its average speedup
collapsing from $4.9\times$ to $3.6\times$ as its per-position marginals diverge
from the target's path-conditioned distribution and are rejected early, whereas
\xpresto{} sustains a $5.3\times$ average speedup ($\tau{=}6.97$) for a
$1.46\times$ throughput gain over dFlash, ranging from $1.23\times$ on LCB to
$1.60\times$ on GSM8K. These gains hold across all three domains: chat
(MT-Bench) is the hardest regime for diffusion drafting in absolute terms, yet
\xpresto{} still delivers $1.36\times$ ($T{=}0$) and $1.39\times$ ($T{=}1$) over
dFlash, indicating the correction helps precisely where the base drafter is
weakest rather than only on easy, high-acceptance workloads.


\textbf{Large Batches on vLLM.} \Cref{tab:vllm-batch} evaluates \xpresto{} inside vLLM's speculative-decoding
stack across batch sizes from $4$ to $32$, where the GPU moves from memory- toward
compute-bound execution and the headroom for speculation necessarily shrinks. As
expected, every method's speedup over the AR baseline compresses as the batch
grows. On GSM8K, \xpresto{} moves from $6.0\times$ at $\text{bs}{=}4$ to
$2.8\times$ at $\text{bs}{=}32$, and dFlash from $4.1\times$ to $2.0\times$,
since AR throughput itself scales nearly linearly with batch size (${\sim}1.8\times$
from $\text{bs}{=}16$ to $\text{bs}{=}32$) while the speculative methods approach
compute saturation. The key observation is that \xpresto{} achieves higher
absolute throughput than both the AR baseline and dFlash in every cell and retains a
meaningful speedup even at $\text{bs}{=}32$ ($2.1$--$2.8\times$ on math and
code), so the refiner's benefit survives into the batched serving regime rather
than vanishing under load. Moreover, its advantage over the plain dFlash
drafter is essentially batch-invariant: \xpresto{} sustains a
$1.16$--$1.47\times$ throughput edge over dFlash across all four batch sizes, and
the edge is nearly constant per dataset (e.g.\ $1.46$--$1.47\times$ on GSM8K,
reaching $13686$ vs.\ $9380$ tok/s at $\text{bs}{=}32$), because its higher acceptance
length ($\tau{=}9.82$ vs.\ $6.02$ on GSM8K) accepts more tokens per
verification and thus amortizes the target forward passes that
dominate cost at scale. 


\begin{figure}[h]
  \centering
  \includegraphics[width=\linewidth]{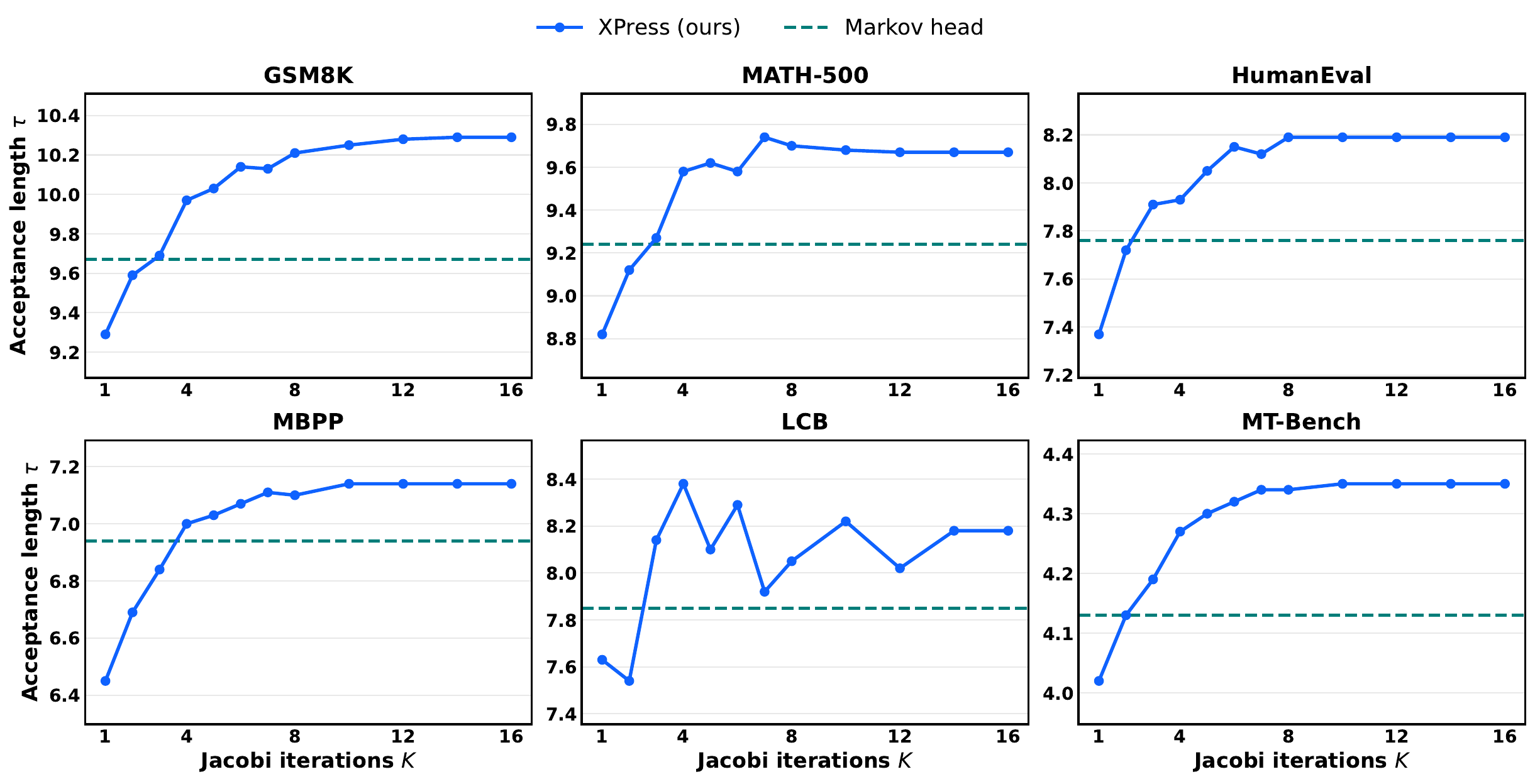}
  \caption{Per-step $\tau$ versus the number of Jacobi iterations $K$, from $K{=}1$ to $16$, with the Markov head as a horizontal baseline (dashed). $\tau$ climbs past the baseline within a few iterations, then flattens; the small non-monotone wiggles past the plateau are dataset-dependent (e.g.\ LiveCodeBench).}
  \label{fig:acck}
\end{figure}

\paragraph{How many iterations are needed?} More Jacobi iterations lock in more of the prefix, but that is not the same as acceptance rising monotonically with $K$. When the accepted prefix already reaches past the settled region, one more iteration can overwrite a not-yet-converged tail token that happened to match the target, so $\tau$ can dip slightly. As shown in Figure~\ref{fig:acck}, we can see that in practice $\tau$ rises quickly, crosses the Markov baseline within a few iterations, and plateaus by $K\approx7$. Pushing to $K{=}16$ never beats the plateau, so a small $K$ captures essentially all of the gain.

\paragraph{Drafting-time latency.} As shown in Table \ref{tab:refine-latency}, \xpresto's cost grows with the number of Jacobi iterations: each Jacobi iteration adds about \SI{75}{\micro\second}, so the refiner runs from \textbf{\SI{150}{\micro\second}} at $K{=}1$ to \textbf{\SI{681}{\micro\second}} at $K{=}8$. The Markov head is a fixed \textbf{\SI{601}{\micro\second}}, due to its fixed 15-step serial decode regardless of $K$, so the two cross near $K{=}7$. As \Cref{fig:acck} shows, \xpresto{} never needs to run that far. It matches or beats the Markov head's per-step acceptance on every benchmark by $K{=}4$, and at $K{=}4$ the refiner costs just \textbf{\SI{379}{\micro\second}}, a $1.6\times$ speedup over the Markov head. In other words, at the first point where \xpresto{} is already more accurate, it is also markedly faster. Pushing on to the accuracy at $K{=}6$ still leaves it cheaper (\SI{530}{\micro\second}, a $1.13\times$ speedup).

\begin{table}[t]
\centering
\small
\setlength{\tabcolsep}{4.5pt}
\caption{Refinement latency per block (\si{\micro\second}): \xpresto{} with $K$
parallel Jacobi passes versus the Markov head's 15-step sequential decode.}
\label{tab:refine-latency}
\begin{tabular}{l cccccccc c}
\toprule
& \multicolumn{8}{c}{\xpresto{}, $K$ Jacobi passes} & Markov head \\
\cmidrule(lr){2-9} \cmidrule(lr){10-10}
& $K{=}1$ & 2 & 3 & 4 & 5 & 6 & 7 & 8 & (15 serial steps) \\
\midrule
Latency (\si{\micro\second})
& 150 & 227 & 301 & 379 & 452 & 530 & 603 & 681 & 601 \\
\bottomrule
\end{tabular}
\end{table}

\Cref{fig:ratio} breaks the draft step into its three parts: the drafter's forward pass, the base \texttt{lm\_head} readout, and the causal refiner. The first two are shared by every head; the only difference is the refiner. 
But that refiner cost is non-trivial:
even run as a compiled unrolled loop, the Markov head's 15-step serial decode is about 24\% of the draft-side cost. 
\xpresto{} acts to minimize this slice, replacing the 15-step serial decode with a handful of Jacobi iterations.

\begin{figure}[t]
  \centering
  \includegraphics[width=0.9\linewidth]{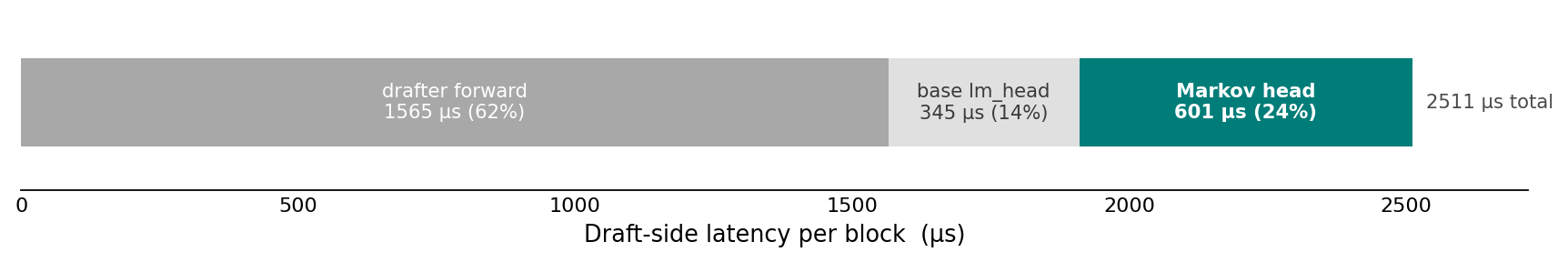}
  \caption{Composition of the draft-side latency per block for the Markov head (GSM8K, block 16, H200/sdpa). The drafter forward and the base \texttt{lm\_head} are shared across heads; the refiner is the CUDA-graphed 15-step serial decode, about a quarter of the draft cost.}
  \label{fig:ratio}
\end{figure}

\section{Conclusion}
\xpresto{} is a lightweight causal refiner for block-diffusion drafters. It reads the drafter's own hidden states and reconciles the whole block through a few parallel Jacobi iterations, restoring the token-to-token dependencies that a parallel drafter drops, without falling back to a serial, left-to-right decode. Across seven math, code, and chat benchmarks on Qwen3-8B, \xpresto{} raises the dFlash drafter's acceptance length by about \textbf{30\%} and its end-to-end throughput by about $1.3\times$. Moreover, by running four Jacobi iterations, \xpresto{} already matches or beats the Markov head's acceptance length while running about \textbf{$1.6\times$} faster.

There are several parts that could be further optimized. The iteration count $K$ is currently fixed for a whole run, but most blocks converge with a small number of iterations, so an adaptive rule that stops a block once it stops changing, or a schedule that spends iterations only where the draft is still unsettled, could cut the average $K$ with no loss in acceptance. More broadly, the philosophy introduced by \xpresto{} is not specific to speculative decoding. Reconciling a block of mutually dependent predictions in a few parallel iterations may help wherever a model emits many interdependent outputs at once.

\bibliographystyle{unsrt}
\bibliography{references}

\end{document}